\documentclass[conference]{IEEEtran}
\IEEEoverridecommandlockouts
\usepackage{cite}
\usepackage{amsmath,amssymb,amsfonts}
\usepackage{algorithmic}
\usepackage{graphicx}
\usepackage{textcomp}
\usepackage{xcolor}
\def\BibTeX{{\rm B\kern-.05em{\sc i\kern-.025em b}\kern-.08em
		T\kern-.1667em\lower.7ex\hbox{E}\kern-.125emX}}

\usepackage{url}

\usepackage{multirow}
\usepackage{multicol}
\usepackage{fdsymbol} 
\usepackage{algorithm}
\usepackage{algorithmic}
\usepackage{subfigure}
\usepackage{centernot}
\usepackage{booktabs}

\makeatletter
\newcommand*{\indep}{%
	\mathbin{%
		\mathpalette{\@indep}{}%
	}%
}
\newcommand*{\nindep}{%
	\mathbin{
		\mathpalette{\@indep}{\not}
	}%
}
\newcommand*{\@indep}[2]{%
	\sbox0{$#1\perp\m@th$}
	\sbox2{$#1=$}
	\sbox4{$#1\vcenter{}$}
	\rlap{\copy0}
	\dimen@=\dimexpr\ht2-\ht4-.2pt\relax
	\kern\dimen@
	{#2}%
	\kern\dimen@
	\copy0 
}

\newcommand{\eat}[1]{}
\newcommand{\etal}{{et al.~}}       
\newcommand{\ie}{{i.e.,~}}           
\newcommand{\wrt}{{wrt.,~}}         

\newtheorem{definition}{Definition}
\newtheorem{problem}{Problem}

\newtheorem{theorem}{Theorem}
\newtheorem{proof}{Proof}

\begin{document}
	
	\title{Disentangled Latent Representation Learning for Tackling the Confounding M-Bias Problem in Causal Inference}
	
\author{
	\IEEEauthorblockN{Debo Cheng$^{\ddagger\dagger}$, Yang Xie$^{\P\dagger}$, Ziqi Xu$^{\ddagger\dagger}$, Jiuyong Li$^{\ddagger*}$, Lin Liu$^\ddagger$, Jixue Liu$^{\ddagger}$,  Yinghao Zhang$^{\P}$ and Zaiwen Feng$^{\P*}$}
	\IEEEauthorblockA{$^\P$ College of Informatics, Huazhong Agricultural University, Wuhan, China}
	\IEEEauthorblockA{$^\ddagger$ UniSA STEM, University of South Australia, Adelaide, Australia}
	}

\maketitle

\let\thefootnote\relax\footnote{\noindent $^\dagger$These authors contributed equally.}
\let\thefootnote\relax\footnote{$^*$Corresponding authors: J. Li (Jiuyong.Li@unisa.edu.au) and Z. Feng (Zaiwen.Feng@mail.hzau.edu.cn).}
	
 	
\begin{abstract}
	In causal inference, it is a fundamental task to estimate the causal effect from observational data. However, latent confounders pose major challenges in causal inference in observational data, for example, confounding bias and $M$-bias.  Recent data-driven causal effect estimators tackle the confounding bias problem via balanced representation learning, but assume no $M$-bias in the system, thus they fail to handle the $M$-bias. 
	In this paper, we identify a challenging and unsolved problem caused by a variable that leads to confounding bias and $M$-bias simultaneously. To address this problem with co-occurring $M$-bias and confounding bias, we propose a novel \underline{D}isentangled \underline{L}atent \underline{R}epresentation learning framework for learning latent representations from proxy variables for unbiased \underline{C}ausal effect \underline{E}stimation (DLRCE) from observational data. Specifically, DLRCE learns three sets of latent representations from the measured proxy variables to adjust for the confounding bias and $M$-bias. Extensive experiments on both synthetic and three real-world datasets demonstrate that DLRCE significantly outperforms the state-of-the-art estimators in the case of the presence of both confounding bias and $M$-bias. 
	\end{abstract}
	
	\begin{IEEEkeywords}
		Causal Inference, Causal Effect Estimation, Confounding Bias, $M$-bias, Disentangled Representation Learning, Latent Confounders 
	\end{IEEEkeywords}
	
	\section{Introduction}
	\label{sec:intro}
	Causal effect estimation is an important approach to understand the underlying causal mechanisms of problems in various areas, such as economics~\cite{imbens2015causal,rubin2005causal}, epidemiology~\cite{greenland2003quantifying}, medicine~\cite{connors1996effectiveness} and computer science~\cite{pearl2009causality,guo2020survey}. Randomised control trials (RCTs) are the gold standard for assessing causal effects, but conducting RCTs can often be infeasible or impractical due to ethical considerations, high costs, or time constraints~\cite{imbens2015causal,deaton2018understanding}. Therefore, estimating causal effects from observational data has emerged as an important alternative strategy of RCTs~\cite{imbens2015causal,guo2020survey,cheng2022data}. 
		
The presence of confounding bias, caused by confounders, creates challenges when using observational data for causal effect estimation~\cite{imbens2015causal,pearl2009causality}. A confounder is a variable that influences both the treatment variable, denoted as $W$, and the outcome variable, denoted as $Y$. Many works~\cite{shalit2017estimating,hassanpour2019counterfactual} consider all measured variables, denoted as $\mathbf{X}$ as the set of confounders as shown in Fig.~\ref{fig:simpledags} (a), in which $\mathbf{X}$ causes both $W$ and $Y$ simultaneously in the causal DAG\footnote{A DAG (directed acyclic graph) is a graph with directed edges only and contains no cycles. More details of graph terminologies can be found in~\cite{pearl2009causality}.}. To address the confounding bias, various methods that employ covariate adjustment~\cite{pearl2009causality,cheng2022data} or  confounding balancing~\cite{shalit2017estimating,athey2018approximate} have been developed, under the unconfoundedness assumption\footnote{The unconfoundedness assumption holds when there are no unmeasured confounders for each pair measured variables~\cite{rubin1974estimating,imbens2015causal}.}~\cite{imbens2015causal,guo2020survey}. 
	For example, Shalit et al.\cite{shalit2017estimating} proposed a representation learning based counterfactual regression framework.
	
	\begin{figure}
		\centering
		\includegraphics[scale=0.426]{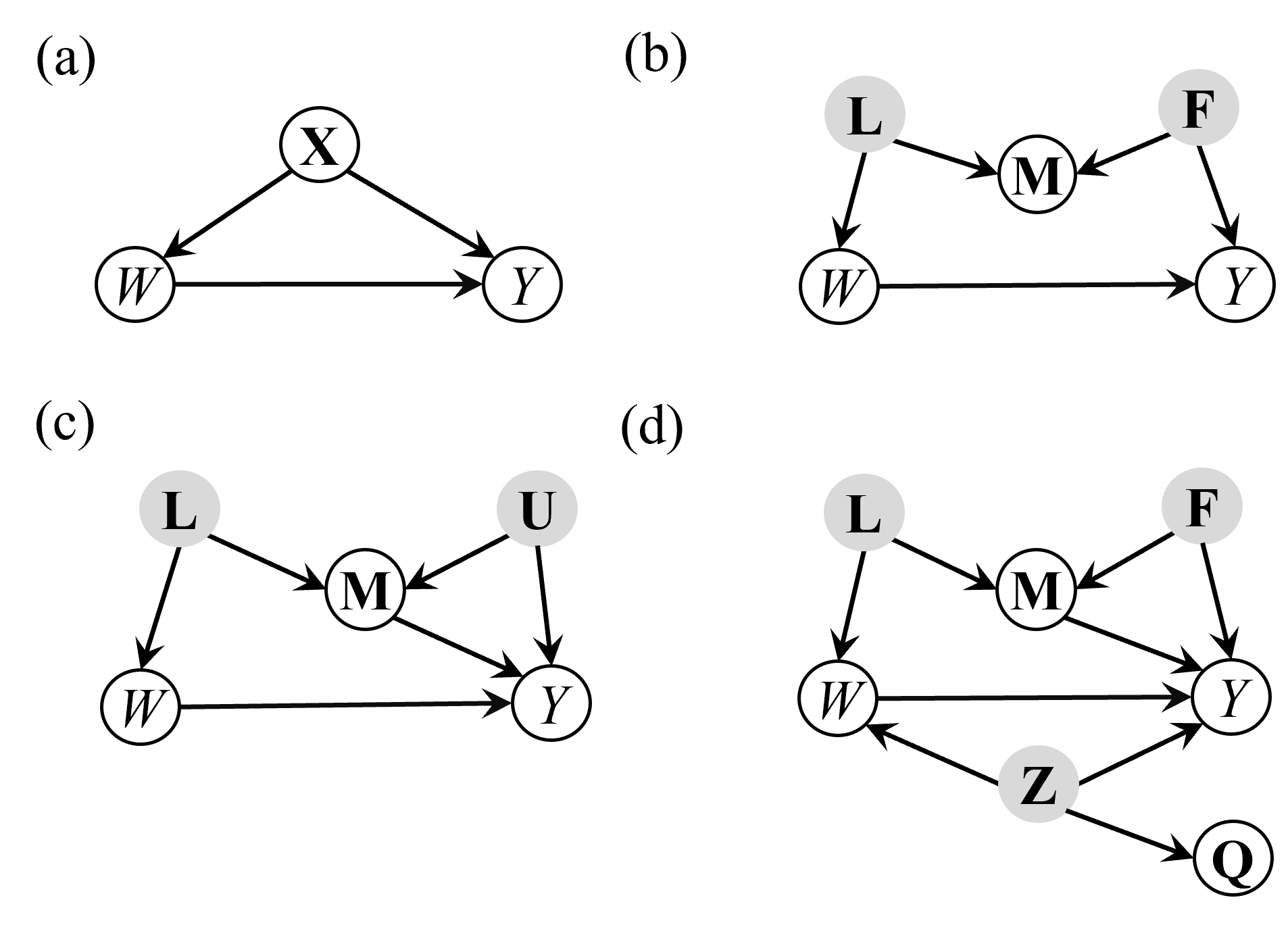}
		\caption{In the figure, $\mathbf{L}$, $\mathbf{F}$  and $\mathbf{Z}$ each represent a set of latent variables (indicated by shaded circles), and the other variables are measured.  $W$ and $Y$ are the treatment and outcome variables respectively. $\mathbf{X}$ is the set of covariates, $\mathbf{M}$ is a set of $M$-bias variables, and $\mathbf{Q}$ is a set of proxy variables of $\mathbf{Z}$. Note that $\mathbf{M}$ is also a set of proxies of $\mathbf{L}$ and $\mathbf{F}$.  In the four causal DAGs,  (a) a simple case with confounding bias caused by $\mathbf{X}$ w.r.t. $(W, Y)$; (b) an illustration of $M$-structure: $W\leftarrow \mathbf{L}\rightarrow \mathbf{M}\leftarrow \mathbf{F}\rightarrow Y$. Conditioning on $\mathbf{M}$ results in $M$-bias \wrt  $(W, Y)$; (c) A DAG illustrating the problem identified in this work, $\mathbf{M}$ serves as a set of $M$-bias variables and a set of confounders, w.r.t. $(W, Y)$; (d) the DAG showing the problem studied in this work.} 
		\label{fig:simpledags}
	\end{figure} 
	
	Nevertheless,  some variables are unmeasured or unobserved due to various uncontrollable factors  in many real-world applications~\cite{perkovic2017complete,van2018separators}, consequently, the unconfoundedness assumption is violated due to the presence of the latent variables. These latent variables result in not only confounding bias, but also $M$-bias. $M$-bias is introduced by conditioning on a variable that is caused by two latent variables.  We call a measured variable an $M$-bias variable if it is a direct effect variable of two or more latent variables, forming an ``$M$-structure''. For example, in Fig.~\ref{fig:simpledags} (b), $\mathbf{M}$ is a set of measured variables and is the set of $M$-bias variables since $\mathbf{L}$ and $\mathbf{F}$ are two sets of latent variables and they are direct causes of $\mathbf{M}$.   
	 In this case, when $\mathbf{M}$ is considered as confounders such that it is adjusted for estimating the causal effect of $W$ on $Y$, a spurious association between $W$ and $Y$ occurs since the path $W \leftarrow \mathbf{L} \rightarrow \mathbf{M}\leftarrow \mathbf{F} \rightarrow Y$ is opened when $\mathbf{M}$ is given (conditioned on). The spurious association along the path causes a biased estimation  \wrt $(W, Y)$. Such a bias is known as $M$-bias in causal effect literature~\cite{greenland2003quantifying,pearl2009myth,ding2015adjust}.

	Excluding the $M$-bias variable in an adjustment set is a common way of dealing with $M$-bias~\cite{greene2003econometric,pearl2009myth,pearl2009causality,cheng2022data}. 
	For example, Enter \etal~\cite{entner2013data} and Cheng \etal~\cite{cheng2022local} use an anchor node to perform conditioning independence/dependence tests for identifying valid adjustment sets to exclude the $M$-bias variable from an adjustment set. 
	
	Dealing with $M$-bias becomes complex when a variable acts as both an $M$-bias variable and a confounder, and we call the problem confounding $M$-bias problem. We call a variable acting both as an $M$-bias variable and a confounder a \emph{confounding $M$-bias variable}, and the problem with a confounding $M$-bias variable the \emph{confounding $M$-bias problem}, and Fig.~\ref{fig:simpledags} (c) shows an example of the problem.   $\mathbf{M}$ is a set of $M$-bias variables based on path $W \leftarrow \mathbf{L} \rightarrow \mathbf{M} \leftarrow \mathbf{F} \rightarrow Y$ and a set of confounders based on path $W \leftarrow \mathbf{L} \rightarrow \mathbf{M} \rightarrow Y$. Using statistical methods, such as Entner \etal's~\cite{entner2013data} and Cheng \etal's~\cite{cheng2022local}, whether adjusting for $\mathbf{M}$ or not, leads to a biased causal effect of $W$ on $Y$. There is no immediate solution to the confounding M-bias problem but the problem is real. We substantiate the example in Fig.~\ref{fig:simpledags} (c) by letting $W$ be `Study time', $Y$ be `Academic performance', $M$ be `Personal interests', $L$ be `Personal experience', and $F$ be `IQ'. When `Personal experience' and `IQ' are unmeasured, it is impossible to estimate the causal effect of `Study time' on `Academic performance' since `Personal interests' is both an $M$-bias variable and a confounder.     
	
	In this paper, we will solve the problem by leveraging the representation learning technique to recover the information of latent variables, and then using observed variables and the learned representations to unbiasedly estimate causal effect in the presence of confounding $M$-bias variables. 
	
	 The causal graph of the problem that is considered in this paper is shown in Fig.~\ref{fig:simpledags} (d). To make our solution covers a broader range of practical scenarios, we also consider the latent confounders whose proxies are observed. For example, in Fig.~\ref{fig:simpledags} (d), if the unobserved confounder $Z$ represents `Teaching quality', and $Q$ represents `Schools', then $Q$ can be used as the proxy of $Z$.

	 In summary, this paper makes the following contributions:  
	
	\begin{itemize}
		\item We identify the confounding $M$-bias problem in causal effect estimation, which is a realistic problem. The problem has not been identified or studied previously. 
		\item We propose a solution, the DLRCE (\underline{D}isentangled \underline{L}atent \underline{R}epresentation learning for unbiased \underline{C}ausal effect \underline{E}stimation) algorithm to resolve the confounding M-bias problem. To the best of our knowledge, there are no solutions to this problem. Furthermore, we prove the soundness of the solution.
		\item We conduct an empirical evaluation to assess the performance of the proposed algorithm on both synthetic and real-world datasets, in comparison to state-of-the-art methods. The experimental results reveal that the proposed algorithm effectively mitigates confounding bias and handles $M$-bias,  and demonstrate its superior performance compared to the baseline methods.   
	\end{itemize}
	
	\section{Preliminaries}
	\label{sec:pre}
	Throughout the paper, we use uppercase and lowercase letters to denote variables and their values, respectively. We use bold-faced uppercase and lowercase letters to represent a set of variables and their corresponding values, respectively.
	
	A graph is a Directed Acyclic Graph (DAG) when it consists of directed edges (represented by $\rightarrow$) and does not contain cycles. In this paper, we use $\mathcal{G}=(\mathbf{V, E})$ to denote a DAG, where $\mathbf{V}=\mathbf{X}\cup\mathbf{U}\cup\{W, Y\}$ represents the set of nodes, 
\ie  $\mathbf{X}$ the set of measured variables, $\mathbf{U}$ the set of latent confounders, $W$ the treatment variable, and $Y$ the outcome variable, and $\mathbf{E}\subseteq \mathbf{V} \times \mathbf{V}$ indicates the set of directed edges. 
	
	In DAG $\mathcal{G}$, two nodes are \textit{adjacent} when there exists a directed edge $\rightarrow$ between them. In a causal DAG, a directed edge $X_i \rightarrow X_j$ signifies that variable $X_i$ is a cause of variable $X_j$ and $X_j$ is an effect variable of $X_i$.  A path $\pi$ from $X_i$ to $X_k$ is a directed or causal path if all edges along it are directed towards $X_k$. If there is a directed path $\pi$ from $X_i$ to $X_k$, $X_i$ is known as an ancestor of $X_k$ and $X_k$ is a descendant of $X_i$. The sets of ancestors and descendants of a node $X$ are denoted as $An(X)$ and $De(X)$, respectively.
	
	A causal DAG $\mathcal{G}=(\mathbf{V, E})$ is employed to represent the underlying causal mechanism of a system. The following presented Markov property and faithfulness assumptions are often assumed in causal inference with a causal DAG.  
	
	\begin{definition}[Markov property~\cite{spirtes2000causation,pearl2009causality}]
		\label{Markov condition}
		Given a DAG $\mathcal{G}=(\mathbf{V}, \mathbf{E})$ and the joint probability distribution $P(\mathbf{V})$, $\mathcal{G}$ satisfies the Markov property if for $\forall V_i \in \mathbf{V}$, $V_i$ is probabilistically independent of all of its non-descendants in $P(\mathbf{V})$, given the parent nodes of $V_i$.
	\end{definition}
	
	\begin{definition}[Faithfulness~\cite{spirtes2000causation}]
		\label{Faithfulness}
		Given a DAG $\mathcal{G}=(\mathbf{V}, \mathbf{E})$ and the joint probability distribution $P(\mathbf{V})$, $\mathcal{G}$ is faithful to a joint distribution $P(\mathbf{V})$ over $\mathbf{V}$ if and only if every independence present in $P(\mathbf{V})$ is entailed by $\mathcal{G}$ and satisfies the Markov property. A joint distribution $P(\mathbf{V})$ over $\mathbf{V}$ is faithful to $\mathcal{G}$ if and only if $\mathcal{G}$ is faithful to $P(\mathbf{V})$.
	\end{definition}

	When the Markov property and faithfulness are satisfied, the dependency/independency relations between variables in the probability distribution $P(\mathbf{V})$ can be inferred from the corresponding causal DAG $\mathcal{G}$~\cite{pearl2009causality,spirtes2000causation}. To determine the conditional independence relationships implied by $\mathcal{G}$, Pearl introduced a graphical criterion, named d-separation.

	\begin{definition}[d-separation~\cite{pearl2009causality}]
		\label{d-separation}
		A path $\pi$ in a DAG $\mathcal{G}=(\mathbf{V}, \mathbf{E})$ is said to be d-separated (or blocked) by a set of nodes $\mathbf{S}$ if and only if
		(i) $\pi$ contains a chain $V_i \rightarrow V_k \rightarrow V_j$ or a fork $V_i \leftarrow V_k \rightarrow V_j$ such that the middle node $V_k$ is in $\mathbf{S}$, or
		(ii) $\pi$ contains a collider $V_k$ such that $V_k$ is not in $\mathbf{S}$ and no descendant of $V_k$ is in $\mathbf{S}$.
		A set $\mathbf{S}$ is said to d-separate $V_i$ from $V_j$ ($ V_i \Vbar V_j\mid\mathbf{S}$) if and only if $\mathbf{S}$ blocks every path between $V_i$ to $V_j$. 
		Otherwise, they are said to be d-connected by $\mathbf{S}$, denoted as $V_i \nVbar V_j\mid\mathbf{S}$.
	\end{definition}

	In this work, we assume that the set $\mathbf{X}$ contains pretreatment variables, \ie all variables in $\mathbf{X}$ are measured before the treatment $W$ is applied and the outcome $Y$ is measured. Based on the potential outcome framework~\cite{rosenbaum1983central,rubin1979using,imbens2015causal}, for an individual, with respect to a binary treatment, there are two potential outcomes, $Y(W=1)$ and $Y(W=0)$ respectively. $Y(W=w)$ is the observed outcome when the treatment $W$ is equal to $w$. Note that, for an individual, we can only observe one of $Y(W=1)$ and $Y(W=0)$ relative to the factual treatment we have applied. The unobserved potential outcome is known as the counterfactual outcome~\cite{rosenbaum1983central,rubin1979using,imbens2015causal}. The individual treatment effect (ITE) for  $i$ is defined as:
	\begin{equation}
		\label{eq:ite}
		ITE_i = Y_i(W=1) - Y_i(W=0)
	\end{equation} 
	
	The average treatment effect (ATE) of $W$ on $Y$ at the population level is defined as:
	
	\begin{equation}
		\label{eq:ate}
		ATE(W, Y)= \mathbb{E}[Y_i(W=1) - Y_i(W=0)]
	\end{equation} 
	\noindent where $\mathbb{E}$ indicates the expectation function. In graphical causal modelling, the ATE is defined as the following using ``do'' operation  introduced by Pearl~\cite{pearl2009causality}, and defined as: 	
	
	\begin{equation}
		\label{eq:ate_do}	
		\begin{aligned}
		&ATE(W, Y) =  \\ & \mathbb{E}[Y\mid do(W=1)] -\mathbb{E}[Y\mid do(W=0)]  
		\end{aligned}
	\end{equation}
	
	\noindent where $do(\cdot)$ is the do-operator.
	
	The ITE defined in Eq.(\ref{eq:ite})  cannot be obtained from data directly since only one potential outcome is observed for an individual. Instead, a number of data-driven methods have been developed for ATE estimation from data. To estimate the causal effect of $W$ on $Y$ unbiasedly from observational data, covariate adjustment~\cite{pearl2009causality,de2011covariate,perkovic2017complete} and confounding balance~\cite{shalit2017estimating} are commonly used method for eliminating the confounding bias. It is critical to discover an adjustment set to eliminate the confounding bias when estimating the causal effect of $W$ on $Y$. The back-door criterion, a well-known graphical criterion, can be applied to discover such an adjustment set $\mathbf{S}\subseteq \mathbf{X}$  in $\mathcal{G}$.
	
\begin{definition}[Back-door criterion~\cite{pearl2009causality}]
		\label{def:backdoorcrite}
		In a DAG $\mathcal{G}=(\mathbf{V}, \mathbf{E})$, for the pair of variables $(W, Y)\in \mathbf{V}$, a set of variables $\mathbf{S}\subseteq \mathbf{V}\setminus\{W, Y\}$ is said to satisfy the back-door criterion in the given DAG $\mathcal{G}$ if (i) $\mathbf{S}$ does not contain a descendant node of $W$; and (ii) $\mathbf{S}$ blocks every back-door path between $W$ and $Y$ (the paths between $W$ and $Y$ starting with an arrow into $W$). A set $\mathbf{S}$ is referred to as a \emph{back-door set} relative to $(W, Y)$ in $\mathcal{G}$ if $\mathbf{S}$ satisfies the back-door criterion relative to $(W, Y)$ in $\mathcal{G}$. Therefore, adjusting for the back-door set $\mathbf{S}$, we have $ATE(W, Y) 	= \mathbb{E}[Y\mid w=1, \mathbf{S}=\mathbf{s}]-  \mathbb{E}[Y\mid w=0, \mathbf{S}=\mathbf{s}]$.
	\end{definition}
	
	For a sub-population with the same features, the conditional ATE (CATE) (Some researchers use it to approximate ITE~\cite{shalit2017estimating}) from observational data is as follows: 
	\begin{equation}
		\label{eq:cate}
		\begin{aligned}
			CATE(W, Y\mid\mathbf{X}=\mathbf{x}) & =  \mathbb{E}[Y\mid do(W=1), \mathbf{X}=\mathbf{x}] \\ & -  \mathbb{E}[Y\mid do(W=0), \mathbf{X}=\mathbf{x}] 
		\end{aligned}
	\end{equation}
	\noindent where $\mathbf{X}$ contains all factors causing the outcome $Y$.
	
	 When there exists a confounding $M$-bias variable, the set of measured variables $\mathbf{X}$ are not enough for the identification of $ATE(W, Y)$ and $CATE(W, Y\mid\mathbf{X}=\mathbf{x})$ from data as discussed in Introduction. We will introduce our DLRCE algorithm for solving this challenging problem in Section~\ref{sec:DLRCE}.
	
	\section{The Proposed DLRCE Algorithm}
	\label{sec:DLRCE}
	\subsection{Problem Setting}
	\label{subsec:prob}
	We assume that the underlying data generation or causal mechanism is represented as a causal DAG $\mathcal{G}=(\mathbf{X}\cup\mathbf{U}\cup\{W, Y\}, \mathbf{E})$  shown in Fig.~\ref{fig:simpledags} (d), where $\mathbf{U} =\mathbf{Z}\cup\mathbf{L}\cup\mathbf{F}$ are latent confounders, $\mathbf{X}= \mathbf{Q}\cup \mathbf{M}$ are measured variables, $\mathbf{Q}$ is the set of proxy variables for $\mathbf{Z}$, and $\mathbf{M}$ is the set of proxy variables for both $\mathbf{L}$ and $\mathbf{F}$. 
	
	Existing methods cannot be used to obtain an unbiased estimation of the causal effect of $W$ on $Y$ using measured variables since either adjusting or not adjusting for $\mathbf{M}$ results in a biased estimation. The aim of this paper is to unbiasedly estimate the ATE of $W$ on $Y$,  and the CATE of $W$ on $Y$ conditioning on $\mathbf{X}$ from observational data. More precisely, the research problem to be tackled in this paper is as follows. 
	
	\begin{problem}
		\label{pro001}
		Given an observational dataset $\mathcal{D}$ of a set measured variables $\{\mathbf{X}=\mathbf{Q}\cup \mathbf{M},  {W, Y}\}$,  and assume that $\mathcal{D}$ is generated from the underlying DAG $\mathcal{G}\!\!=\!\!(\mathbf{X}\cup \mathbf{U}\cup \{W, Y\}, \mathbf{E})$ as shown in Fig.~\ref{fig:simpledags} (d).  $W$ and $Y$ are the treatment and outcome variables respectively and $\mathbf{X}$ contains pretreatment variables. $\mathbf{Q}$ is a set of proxy confounders and $\mathbf{M}$ represents confounding $M$-bias variables. The goal is to estimate $ATE(W, Y)$ and $CATE(W, Y\mid \mathbf{X}=\mathbf{x})$ from the dataset $\mathcal{D}$.
	\end{problem} 

	\subsection{Theoretical base  of the Proposed DLRCE Algorithm}
	\label{subsec:DLRCE}
	We will leverage the capability of VAE (variational autoencoder)~\cite{kingma2014auto,kingma2019introduction} in disentangled representation learning to tackle the confounding $M$-bias problem. VAEs are generative models, also known as latent variable models, and use a prior distribution and a noise distribution for generating the latent representations of the measured variables. We propose to make use of the VAE technique to learn and disentangle the latent representations of the latent variables in our problem setting. Specifically, we propose to use VAE to learn the representation of $\mathbf{Z}$ from the proxy variables $\mathbf{Q}$, and the latent representations $\mathbf{\Psi}$ from $\mathbf{M}$, and then disentangle $\mathbf{\Psi}$ into representations of $\mathbf{L}$ and $\mathbf{F}$, respectively. As in VAE literature, we use the same letter to denote a latent variable and its learned representation. The learned and disentangled latent representations $\{\mathbf{F}, \mathbf{Z}\}$ and $\mathbf{M}$ are used to obtain unbiased estimation of $CATE(W, Y\mid \mathbf{X}=\mathbf{x})$ and $ATE(W, Y)$ from observational data with latent confounders.

	We first demonstrate that the latent representations learned and disentangled by the DLRCE algorithm are sound to estimate $ATE(W, Y)$ from  the dataset $\mathcal{D}$. 	
	
	\begin{theorem}
		\label{theorem:001}
		Given the setting in Problem~\ref{pro001}, $ATE(W, Y)$ can be identified if the latent representations $\mathbf{Z}$ and $\mathbf{L}$ are recovered from the dataset $\mathcal{D}$, and we have $ATE(W, Y) = \mathbb{E}[Y\mid W=1, \mathbf{Z}=\mathbf{z}, \mathbf{L}=\mathbf{l}]-  \mathbb{E}[Y\mid W=0, \mathbf{Z}=\mathbf{z}, \mathbf{L}=\mathbf{l}]$.
	\end{theorem}
	\begin{proof}
		In Fig.~\ref{fig:simpledags} (d), $\mathbf{Z}$ and $\mathbf{L}$ are the parents of $W$. We will prove that $\mathbf{Z}\cup \mathbf{L}$ satisfies the back-door criterion (Definition~\ref{def:backdoorcrite}) \wrt $(W, Y)$, \ie $\mathbf{Z}\cup \mathbf{L}$ blocks all the backdoor paths between $W$ and $Y$. Firstly, $\mathbf{Z}$ and $\mathbf{L}$ do not contain any descendants of $W$, \ie the first condition (i) of the back-door criterion holds. Secondly, there are three back-door paths between $W$ and $Y$, i.e., $W\leftarrow \mathbf{Z}\rightarrow Y$,  $W\leftarrow \mathbf{L}\rightarrow \mathbf{M}\rightarrow Y$, and $W\leftarrow \mathbf{L}\rightarrow \mathbf{M}\leftarrow \mathbf{F}\rightarrow Y$. The first back-door path is blocked by $\mathbf{Z}$ and the remaining two back-door paths are blocked by $\mathbf{L}$. Hence $\mathbf{Z}$ and $\mathbf{L}$ block all back-door paths between $W$ and $Y$, \ie the second condition (ii) of the back-door criterion holds. Therefore,  $\mathbf{Z}\cup \mathbf{L}$ satisfies the back-door criterion and based on Eq.~\ref{eq:ate_do} $ATE(W, Y)$ can be identified in the dataset $\mathcal{D}$. Hence, $ATE(W, Y) = \mathbb{E}[Y\mid W=1, \mathbf{Z}=\mathbf{z}, \mathbf{L}=\mathbf{l}]-  \mathbb{E}[Y\mid W=0, \mathbf{Z}=\mathbf{z}, \mathbf{L}=\mathbf{l}]$.  
	\end{proof}

	Theorem~\ref{theorem:001} presents a theoretical base for $ATE(W, Y)$ estimation. It is worth mentioning that the conditional clause `if  $\mathbf{Z}$ and $\mathbf{L}$ are recovered from the dataset $\mathcal{D}$' in the theorem is a fundamental assumption that is widely made in VAE-based causal inference~\cite{louizos2017causal,hassanpour2019learning,zhang2021treatment}.

	In the following theorem, we will show that the latent representations learned and disentangled by the DLRCE  are  
	sound to estimate $CATE(W, Y\mid \mathbf{X}=\mathbf{x})$ from $\mathcal{D}$.
	\begin{theorem}
		\label{theorem:002}
			Given the setting in Problem~\ref{pro001}, $CATE(W, Y\mid\mathbf{X}=\mathbf{x})$ can be identified if the latent variables $\mathbf{Z}$ and $\mathbf{F}$ are recovered from the dataset $\mathcal{D}$, and we have $CATE(W, Y\mid \mathbf{X}=\mathbf{x}) = \mathbb{E}[Y\mid W=1, \mathbf{Z}=\mathbf{z}, \mathbf{M}=\mathbf{m}, \mathbf{F}=\mathbf{f}]-  \mathbb{E}[Y\mid W=0, \mathbf{Z}=\mathbf{z}, \mathbf{M}=\mathbf{m}, \mathbf{F}=\mathbf{f}]$.
	\end{theorem}
	\begin{proof}
		We first use the `do' calculus rules~\cite{pearl2009causality} to remove the `do' operator from the definition of $CATE$, \ie $CATE(W, Y\mid \mathbf{X}=\mathbf{x})= \mathbb{E}[Y\mid do(W), \mathbf{Z}=\mathbf{z}, \mathbf{F}=\mathbf{f}, \mathbf{L}=\mathbf{l}, \mathbf{X}=\mathbf{x}] = \mathbb{E}[Y\mid do(W), \mathbf{Z}=\mathbf{z}, \mathbf{F}=\mathbf{f}, \mathbf{L}=\mathbf{l}, \mathbf{M}=\mathbf{m}, \mathbf{Q}=\mathbf{q}]$. Let $\mathcal{G}_{\underline{W}}$ be the manipulated DAG by removing all outgoing edges of $W$ from the causal DAG in Fig.~\ref{fig:simpledags} (d), and $\mathcal{G}_{\overline{W}}$ represents the manipulated DAG by eliminating all edges into $W$. Note that $Y\Vbar \mathbf{L}\mid \mathbf{M},\mathbf{F}$ and $Y\Vbar \mathbf{Q}\mid \mathbf{Z}$ in $\mathcal{G}_{\overline{W}}$. Hence, using Rule 3 of do-calculus, we can remove $\mathbf{L}$ and $\mathbf{Q}$ from the conditioning set, and obtain $CATE(W, Y\mid \mathbf{X}=\mathbf{x}) = \mathbb{E}[Y\mid do(W), \mathbf{Z}=\mathbf{z}, \mathbf{F}=\mathbf{f}, \mathbf{M}=\mathbf{m}]$. Following Rule 2 of do-calculus~\cite{pearl2009causality} with the condition $Y\Vbar W\mid \mathbf{Z, M, F}$ in $\mathcal{G}_{\underline{W}}$, we have $CATE(W, Y\mid \mathbf{X}=\mathbf{x})= \mathbb{E}[Y\mid do(W), \mathbf{F = f, M = m, Z = z}) = \mathbb{E}[Y\mid W, \mathbf{Z=z, M=m, F=f}]$. Therefore, $\mathbf{Z}$, $\mathbf{F}$ and $\mathbf{M}$ are sufficient for identifying $CATE(W, Y\mid\mathbf{X}=\mathbf{x})$ from the dataset $\mathcal{D}$. Hence, $CATE(W, Y\mid \mathbf{X}=\mathbf{x}) = \mathbb{E}[Y\mid W, \mathbf{Z=z, M=m, F=f}]$. 
	\end{proof}
	
	Theorems~\ref{theorem:001} and~\ref{theorem:002} provide the ground that learning the latent representations allows us to unbiasedly estimate $CATE$ and $ATE$ from observational data when there exists a set of proxy variables, \ie $\mathbf{X}=(\mathbf{M, Q})$.  In the next section, we will introduce our proposed DLRCE algorithm for learning these latent representations from observational data.
	
	\subsection{Disentanglement of Latent Representations}
	\label{subsec:model}
	In this section, we present the details of our proposed DLRCE algorithm that is built on disentangled representation learning and supported by Theorems 1 and 2. The overall architecture of DLRCE is presented in Fig.~\ref{fig:archit}. DLRCE aims to learn the latent representations $\mathbf{Z}$ from the proxy variables $\mathbf{Q}$, and the latent representations $\mathbf{\Psi}$ from the proxy variables $\mathbf{M}$, and then disentangle $\mathbf{\Psi}$ into two disjoint sets $\mathbf{L}$ and $\mathbf{F}$. Finally, the latent representations $\{\mathbf{F}, \mathbf{Z}\}$ and $\mathbf{M}$ are used for calculating causal effects of $W$ on $Y$.
	
	\begin{figure*}[t]
		\centering
		\includegraphics[scale=0.45]{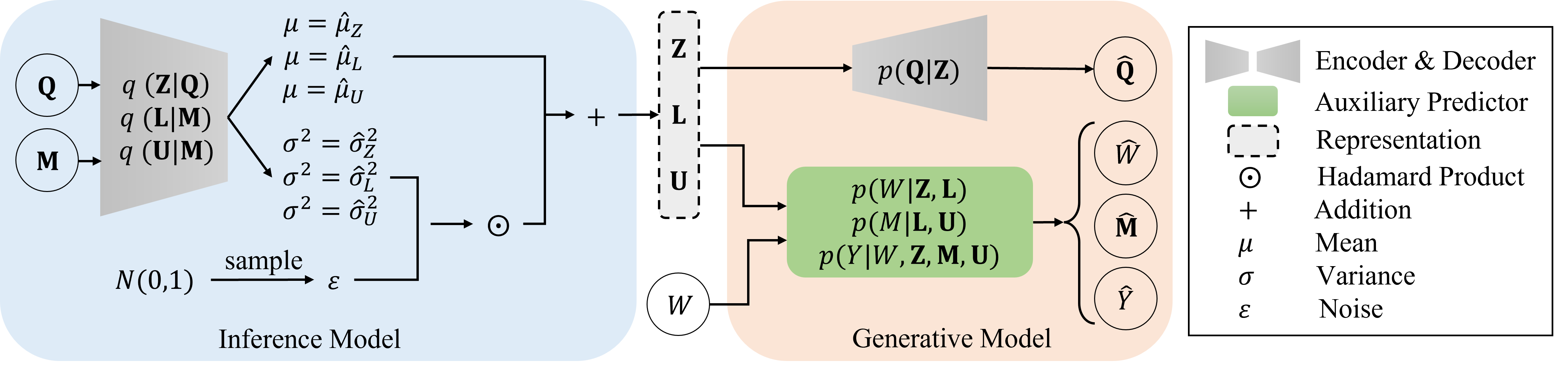}
		\caption{The architecture of the proposed DLRCE algorithm consists of the inference network and the generative network for learning the three representations from proxy variables. Three auxiliary predictors ensure that the treatment $W$ is predicted by $\mathbf{Z}$ and $\mathbf{L}$, the measured variables $\mathbf{M}$ are predicted by $\mathbf{L}$ and $\mathbf{F}$, and the outcome $Y$ is predicted  by $\mathbf{Z}$, $\mathbf{L}$ and $\mathbf{M}$.}
		\label{fig:archit}
	\end{figure*} 
	
	To learn and disentangle the representations, the inference model and the generative model are employed to approximate the two posteriors $p_{\varphi_{\mathbf{Q}}}(\mathbf{Q}\mid\mathbf{Z})$ and  $p_{\varphi_{\mathbf{M}}}(\mathbf{M}\mid\mathbf{L}, \mathbf{F})$, where $\varphi_{\mathbf{Q}}$ and $\varphi_{\mathbf{M}}$ are the network parameters in the generative model. In the inference model of DLRCE, three separate encoders $q_{\theta_\mathbf{Z}}(\mathbf{Z}\mid\mathbf{Q})$, $q_{\theta_\mathbf{L}}(\mathbf{L}\mid\mathbf{M})$ and $q_{\theta_\mathbf{F}}(\mathbf{F}\mid\mathbf{M})$ are employed to serve as variational posteriors for deducing the three latent representations, for which $\theta_\mathbf{Z}$, $\theta_\mathbf{L}$ and $\theta_\mathbf{F}$ are the network parameters. In the generative model of the DLRCE algorithm, the latent representation $\mathbf{Z}$ is generated from a single encoder $q_{\theta_\mathbf{Z}}(\mathbf{Z}\mid\mathbf{Q})$ used by a single decoder $p_{\varphi_{\mathbf{Q}}}(\mathbf{Q}\mid\mathbf{Z})$ to reconstruct $\mathbf{Q}$; the latent representations $\mathbf{L}$ and $\mathbf{F}$ are generated from two separated encoders $q_{\theta_\mathbf{L}}(\mathbf{L}\mid\mathbf{M})$ and $q_{\theta_\mathbf{F}}(\mathbf{F}\mid\mathbf{M})$ used by a single decoder $p_{\varphi_{\mathbf{M}}}(\mathbf{M}\mid\mathbf{L}, \mathbf{F})$ to reconstruct $\mathbf{M}$. 
	
	As in the standard VAE~\cite{kingma2014auto,kingma2019introduction}, we use Gaussian distributions to initialise the prior distributions of $P(\mathbf{Z})$, $P(\mathbf{L})$ and $P(\mathbf{F})$, and defined as:
	\begin{equation}
		\label{eq:1}
		\begin{aligned}
		&	 P(\mathbf{Z}) = \prod_{i=1}^{\left|\mathbf{Z} \right|} \mathcal{N}(Z_{i} \mid 0,1);  
			 P(\mathbf{L}) = \prod_{i=1}^{\left|\mathbf{L} \right|} \mathcal{N}(L_{i} \mid 0,1);
			 \\& P(\mathbf{F}) = \prod_{i=1}^{\left|\mathbf{F} \right|} \mathcal{N}(F_{i} \mid 0,1);
		\end{aligned}	
	\end{equation} 
	
	In the inference model of DLRCE, the variational posteriors for approximating  $q_{\theta_\mathbf{Z}}(\mathbf{Z}\mid\mathbf{Q})$, $q_{\theta_\mathbf{L}}(\mathbf{L}\mid\mathbf{M})$ and $q_{\theta_\mathbf{F}}(\mathbf{F}\mid\mathbf{M})$ are defined as:
	\begin{equation}
		\label{eq:2}
		\begin{aligned}
			&q_{\theta_\mathbf{Z}}(\mathbf{Z}\mid\mathbf{Q}) = \prod_{i=1}^{\left|\mathbf{Z} \right|} \mathcal{N}(\mu = \hat{\mu}_{Z_i}, \sigma^2 = \hat{\sigma}^2_{Z_i}); 
			\\& q_{\theta_\mathbf{L}}(\mathbf{L}\mid\mathbf{M}) = \prod_{i=1}^{\left|\mathbf{L}\right|} \mathcal{N}(\mu = \hat{\mu}_{L_i}, \sigma^2 = \hat{\sigma}^2_{L_i}); 
			\\& q_{\theta_\mathbf{F}}(\mathbf{F}\mid\mathbf{M}) = \prod_{i=1}^{\left|\mathbf{F}\right|} \mathcal{N}(\mu = \hat{\mu}_{F_i}, \sigma^2 = \hat{\sigma}^2_{F_i});
		\end{aligned}	
	\end{equation}
	\noindent where $\hat{\mu}_{Z_i}, \hat{\mu}_{L_i}, \hat{\mu}_{F_i}$ and $\hat{\sigma}^2_{Z_i}, \hat{\sigma}^2_{L_i}, \hat{\sigma}^2_{F_i}$ are the estimated means and variances of latent variables $Z_i$, $L_i$ and $F_i$, respectively. 
	
	In the generative models of DLRCE, $\mathbf{Z}$, $\mathbf{L}$ and $\mathbf{F}$ are used to generate $W$, $\mathbf{Q}$ and $\mathbf{M}$ as:
	\begin{equation}
		\label{eq:3}
		\begin{aligned}
			&p_{\varphi_{W}}(W\mid\mathbf{Z}, \mathbf{L}) = B(\sigma(g_1(\mathbf{Z}, \mathbf{L})));  
			\\& p_{\varphi_{\mathbf{Q}}}(\mathbf{Q}\mid\mathbf{Z}) = \prod_{i=1}^{\left|\mathbf{Q} \right|} P(Q_i \mid\mathbf{Z});
			\\& p_{\varphi_{\mathbf{M}}}(\mathbf{M}\mid\mathbf{L}, \mathbf{F}) = \prod_{i=1}^{\left|\mathbf{M} \right|} P(M_i \mid\mathbf{L}, \mathbf{F});
		\end{aligned}	
	\end{equation} 
	\noindent where $B(\cdot)$ is the Bernoulli function, $g_1(\cdot)$ is a neural network, $\sigma(\cdot)$ is the logistic function. Notably, $P(Q_i \mid\mathbf{Z})$ and $ P(M_i \mid\mathbf{L}, \mathbf{F})$ are the distributions on the $i$-th variable.
	
	In the generative model of DLRCE, we generate the outcome $Y$ based on its data type. 
	
	 When $Y$ is a continuous variable, we sample $Y$ from a normal distribution and model it in the treatment and control groups as $P(Y \mid W = 0, \mathbf{Z}, \mathbf{M}, \mathbf{F})$ and $P(Y \mid W = 1, \mathbf{Z}, \mathbf{M}, \mathbf{F})$ respectively. Hence, the generative model for $Y$ is described as:
	\begin{equation}
		\label{eq:4}
		\begin{aligned}
			& p_{\varphi_{Y}}(Y \mid W, \mathbf{Z}, \mathbf{M}, \mathbf{F}) = \mathcal{N}(\mu = \hat{\mu}_{Y}, \sigma^2 = \hat{\sigma}^2_{Y}); \\
			&\hat{\mu}_{Y} = W \cdot g_2( \mathbf{Z}, \mathbf{M}, \mathbf{F}) + (1-W) \cdot g_3( \mathbf{Z}, \mathbf{M}, \mathbf{F});  \\ 
			&\hat{\sigma}^2_{Y} = W \cdot g_4( \mathbf{Z}, \mathbf{M}, \mathbf{F}) + (1-W) \cdot g_5( \mathbf{Z}, \mathbf{M}, \mathbf{F});
		\end{aligned}	
	\end{equation}
	\noindent where $g_2(\cdot)$, $g_3(\cdot)$, $g_4(\cdot)$ and $g_5(\cdot)$ are the functions parameterised by neural networks. 
	
	When $Y$ is a binary variable, we use a Bernoulli distribution to parameterise it:
	\begin{equation}
		\label{eq:5}
		\begin{aligned}
			&p_{\varphi_{Y}}(Y\mid W, \mathbf{Z}, \mathbf{M}, \mathbf{F}) = B(\sigma(g_6(W, \mathbf{Z}, \mathbf{M}, \mathbf{F}))),
		\end{aligned}	
	\end{equation}
	\noindent where $g_6(\cdot)$ is a neural network. We optimise these parameters by maximising the evidence lower bound (ELBO) ~\cite{kingma2014auto}:
	\begin{equation}
		\label{eq:6}
		\begin{aligned}
			\mathcal{L}_{ELBO} = \: &\mathbb{E}_{q_{\theta_{\mathbf{Z}}}}[\log p_{\varphi_{\mathbf{Q}}}(\mathbf{Q}\mid\mathbf{Z})]+ \\&
			\mathbb{E}_{q_{\theta_{\mathbf{L}}}q_{\theta_{\mathbf{F}}}}[\log p_{\varphi_{\mathbf{M}}}(\mathbf{M}\mid\mathbf{L}, \mathbf{F})] \\& -  D_{KL}[q_{\theta_{\mathbf{Z}}}(\mathbf{Z}\mid \mathbf{Q})||P(\mathbf{Z})]  \\ 
			& - D_{KL}[q_{\theta_{\mathbf{L}}}(\mathbf{L}\mid\mathbf{M})||P(\mathbf{L})] \\&
			- D_{KL}[q_{\theta_{\mathbf{F}}}(\mathbf{F}\mid\mathbf{M})||P(\mathbf{F})],
		\end{aligned}	
	\end{equation}
	\noindent where $D_{KL}[\cdot||\cdot]$ is a KL divergence term. 
	
	To encourage the disentanglement of latent representations and ensure that $\mathbf{M}$ can be recovered by $\mathbf{L}$ and $\mathbf{F}$, and to ensure $W$ can be predicted by $\mathbf{Z}$ and $\mathbf{L}$, and $Y$ can be predicted by $\mathbf{Z}$, $\mathbf{M}$ and $\mathbf{F}$, three auxiliary predictors are added to the variational ELBO. Finally, the objective of DLRCE can be described as:
	\begin{equation}
		\label{eq:7}
		\begin{aligned}
			\mathcal{L}_{DLRCE} =\: & -\mathcal{L}_{ELBO} +  
			\alpha \mathbb{E}_{q_{\theta_{\mathbf{L}}}q_{\theta_{\mathbf{F}}}}[\log q(\mathbf{M}\mid\mathbf{L}, \mathbf{F})]\\&~ + \beta \mathbb{E}_{q_{\theta_{\mathbf{Z}}}q_{\theta_{\mathbf{L}}}}[\log q(W\mid\mathbf{Z}, \mathbf{L})] \\& ~ +
			\gamma \mathbb{E}_{q_{\theta_{\mathbf{Z}}}q_{\theta_{\mathbf{F}}}}[\log 
			q(Y\mid W, \mathbf{Z}, \mathbf{M}, \mathbf{F})],
		\end{aligned}	
	\end{equation}
	\noindent where $\alpha$, $\beta$ and $\gamma$ are the weights for balancing the three auxiliary predictors. 
	
	To estimate the CATEs of individuals conditioning on their measured variables $\mathbf{X}$, we employ the three encoders $q_{\theta_\mathbf{Z}}(\mathbf{Z}\mid\mathbf{Q})$, $q_{\theta_\mathbf{L}}(\mathbf{L}\mid\mathbf{M})$ and $q_{\theta_\mathbf{F}}(\mathbf{F}\mid\mathbf{M})$ to sample the approximated posteriors, and average the predicted potential outcomes using the classifier $q(Y\mid W, \mathbf{Z}, \mathbf{M}, \mathbf{F})$. Finally, by utilising Theorems 1 and 2, DLRCE is able to estimate the $ATE(W, Y)$ and $CATE(W, Y\mid \mathbf{X}=\mathbf{x})$ from the dataset $\mathcal{D}$.

	\section{Experiments}
	In this section, we conduct experiments on both synthetic and real-world datasets to evaluate the performance of DLRCE for estimating $ATE$ and $CATE$ from observational data with latent confounders. For the synthetic datasets, we use the causal DAG in Fig.~\ref{fig:simpledags} (d) to generate synthetic datasets with ground truths of $ATE$ and $CATE$ for evaluating the performance of DLRCE. For the experiments on real-world datasets, we choose three benchmark datasets, Schoolingreturns~\cite{card1993using}, Cattaneo2~\cite{cattaneo2010efficient} and Sachs~\cite{sachs2005causal} where the empirical causal effects are available in the literature.

	\begin{table*}[t]
		\centering
		\setlength\tabcolsep{12pt}
		\caption{Estimation bias (mean$\pm$standard deviation) over 30  independently repeated experiments on the synthetic datasets with $Y_{linear}$. The best result is marked in boldface. Our proposed DLRCE algorithm obtains the smallest bias.}
		\begin{tabular}{|c|ccccc|}
			\hline
			\multirow{2}{*}{Method}   & \multicolumn{5}{c|}{Sample sizes}        \\ \cline{2-6} 
			& \multicolumn{1}{c|}{2k} & \multicolumn{1}{c|}{4k}   & \multicolumn{1}{c|}{6k}   & \multicolumn{1}{c|}{8k}         & 10k              \\ \hline
			LDML                             & \multicolumn{1}{c|}{30.82$\pm$0.31}  & \multicolumn{1}{c|}{{28.56$\pm$0.19}} & \multicolumn{1}{c|}{28.36$\pm$0.17} & \multicolumn{1}{c|}{28.65$\pm$0.12} & 28.08$\pm$0.05                      \\ \hline
			SLDML                       & \multicolumn{1}{c|}{39.98$\pm$0.38}  & \multicolumn{1}{c|}{28.57$\pm$0.18} & \multicolumn{1}{c|}{28.50$\pm$0.16} & \multicolumn{1}{c|}{28.54$\pm$0.11} & 28.10$\pm$0.05                      \\ \hline
			KernelDML                             & \multicolumn{1}{c|}{39.71$\pm$0.46}  & \multicolumn{1}{c|}{41.06$\pm$0.19} & \multicolumn{1}{c|}{41.00$\pm$0.22} & \multicolumn{1}{c|}{41.98$\pm$0.14} & 43.09$\pm$0.09                      \\ \hline
		X-learner                       & \multicolumn{1}{c|}{23.99$\pm$0.51}  & \multicolumn{1}{c|}{22.24$\pm$0.17}   & \multicolumn{1}{c|}{22.15$\pm$0.14} & \multicolumn{1}{c|}{21.86$\pm$0.11} & 21.92$\pm$0.10 \\ \hline
		R-learner                       & \multicolumn{1}{c|}{28.37$\pm$1.04}  & \multicolumn{1}{c|}{39.73$\pm$0.30}   & \multicolumn{1}{c|}{29.04$\pm$0.21} & \multicolumn{1}{c|}{29.95$\pm$0.13} & 28.86$\pm$0.05 \\ \hline
LDRlearner                       & \multicolumn{1}{c|}{48.57$\pm$0.51}  & \multicolumn{1}{c|}{47.48$\pm$0.17}   & \multicolumn{1}{c|}{46.50$\pm$0.18} & \multicolumn{1}{c|}{46.73$\pm$0.12} & 47.22$\pm$0.09 \\ \hline
			CFDML                       & \multicolumn{1}{c|}{39.53$\pm$0.49}  & \multicolumn{1}{c|}{35.51$\pm$0.15} & \multicolumn{1}{c|}{33.16$\pm$0.12} & \multicolumn{1}{c|}{32.50$\pm$0.08} & 31.81$\pm$0.05                      \\ \hline
			CEVAE                                 & \multicolumn{1}{c|}{31.03$\pm$0.79}  & \multicolumn{1}{c|}{45.73$\pm$0.39} & \multicolumn{1}{c|}{35.60$\pm$0.56} & \multicolumn{1}{c|}{29.47$\pm$1.10} & 23.27$\pm$0.68                      \\ \hline			
			TEDVAE                                & \multicolumn{1}{c|}{40.59$\pm$0.40}  & \multicolumn{1}{c|}{34.59$\pm$0.18} & \multicolumn{1}{c|}{31.82$\pm$0.14} & \multicolumn{1}{c|}{31.27$\pm$0.11} & 29.75$\pm$0.07                      \\ \hline			
			DLRCE                                 & \multicolumn{1}{c|}{\textbf{12.62$\pm$1.20}} & \multicolumn{1}{c|}{\textbf{14.09$\pm$1.06}} & \multicolumn{1}{c|}{\textbf{15.20$\pm$1.23}} & \multicolumn{1}{c|}{\textbf{12.54$\pm$0.74}} & \textbf{13.59$\pm$0.54}                      \\ \hline			
		\end{tabular}
		\label{tab:bias001}
	\end{table*}
	
\begin{table*}[t]
		\centering
		\setlength\tabcolsep{12pt}
		\caption{Estimation PEHE (mean$\pm$standard deviation) over 30 independently repeated experiments on the synthetic datasets with $Y_{linear}$. The best result is marked in boldface. Our proposed DLRCE algorithm obtains the smallest PEHE.}
		\begin{tabular}{|c|ccccc|}
			\hline
			\multirow{2}{*}{Method}   & \multicolumn{5}{c|}{Sample sizes}         \\ \cline{2-6} 
			& \multicolumn{1}{c|}{2k} & \multicolumn{1}{c|}{4k}   & \multicolumn{1}{c|}{6k}   & \multicolumn{1}{c|}{8k}         & 10k              \\ \hline
			LDML        & \multicolumn{1}{c|}{1.06$\pm$0.03}  & \multicolumn{1}{c|}{0.93$\pm$0.02} & \multicolumn{1}{c|}{0.90$\pm$0.01} & \multicolumn{1}{c|}{0.91$\pm$0.01} & 0.88$\pm$0.01                      \\ \hline
			SLDML  & \multicolumn{1}{c|}{1.07$\pm$0.03}  & \multicolumn{1}{c|}{0.93$\pm$0.02} & \multicolumn{1}{c|}{0.91$\pm$0.01} & \multicolumn{1}{c|}{0.90$\pm$0.01} & 0.89$\pm$0.01                      \\ \hline
			KernelDML        & \multicolumn{1}{c|}{1.25$\pm$0.05}  & \multicolumn{1}{c|}{1.22$\pm$0.02} & \multicolumn{1}{c|}{1.24$\pm$0.02} & \multicolumn{1}{c|}{1.27$\pm$0.01} & 1.48$\pm$0.01                      \\ \hline
X-learner  & \multicolumn{1}{c|}{3.61$\pm$0.01}  & \multicolumn{1}{c|}{3.56$\pm$0.01}   & \multicolumn{1}{c|}{3.53$\pm$0.01} & \multicolumn{1}{c|}{3.52$\pm$0.01} & 3.52$\pm$0.01 \\ \hline
R-learner  & \multicolumn{1}{c|}{7.69$\pm$2.21}  & \multicolumn{1}{c|}{5.00$\pm$0.49}   & \multicolumn{1}{c|}{4.13$\pm$0.37} & \multicolumn{1}{c|}{3.91$\pm$0.25} & 3.39$\pm$0.14 \\ \hline
			LDRlearner  & \multicolumn{1}{c|}{1.61$\pm$0.04}  & \multicolumn{1}{c|}{1.51$\pm$0.01}   & \multicolumn{1}{c|}{1.46$\pm$0.01} & \multicolumn{1}{c|}{1.46$\pm$0.01} & 1.48$\pm$0.01 \\ \hline
			CFDML  & \multicolumn{1}{c|}{1.34$\pm$0.03}  & \multicolumn{1}{c|}{1.19$\pm$0.01} & \multicolumn{1}{c|}{1.15$\pm$0.01} & \multicolumn{1}{c|}{1.12$\pm$0.01} & 1.10$\pm$0.01                      \\ \hline
			CEVAE            & \multicolumn{1}{c|}{1.17$\pm$0.06}  & \multicolumn{1}{c|}{1.36$\pm$0.05} & \multicolumn{1}{c|}{1.18$\pm$0.06} & \multicolumn{1}{c|}{0.94$\pm$0.06} & 0.82$\pm$0.07                      \\ \hline
			TEDVAE           & \multicolumn{1}{c|}{1.39$\pm$0.03}  & \multicolumn{1}{c|}{1.09$\pm$0.01} & \multicolumn{1}{c|}{0.99$\pm$0.01} & \multicolumn{1}{c|}{0.95$\pm$0.01} & 0.94$\pm$0.01                      \\ \hline
			DLRCE            & \multicolumn{1}{c|}{\textbf{0.46$\pm$0.12}} & \multicolumn{1}{c|}{\textbf{0.50$\pm$0.12}} & \multicolumn{1}{c|}{\textbf{0.58$\pm$0.13}} & \multicolumn{1}{c|}{\textbf{0.47$\pm$0.07}} & \textbf{0.56$\pm$0.10}                     \\ \hline		
		\end{tabular}
		\label{tab:pehe001}
	\end{table*}	
	
\subsection{Experiment Setup}
\noindent	\textbf{Baseline causal effect estimators}. We compare our proposed DLRCE algorithm with nine state-of-the-art causal effect estimators that are widely used to estimate ATE and CATE from observational data. The seven estimators can be divided into two groups, Machine Learning based estimators and VAE based estimators. The Machine learning based estimators include (1) LinearDML (LDML)~\cite{chernozhukov2018double}: It is to solve the reverse causal metric bias by applying a cross-fitting strategy; (2) SparseLinearDML (SLDML)~\cite{chernozhukov2017orthogonal}: The loss function of the LinearDML estimator is modified by incorporating $L_1$ regularisation; (3) KernelDML~\cite{nie2021quasi}: It combines dimensionality reduction techniques and kernel methods; (4) Mete-learners (including X-learner and R-learner)~\cite{kunzel2019metalearners}; (5) LinearDRLearner (LDRlearner)~\cite{foster2019orthogonal}: It is based on double neural networks for addressing the bias in causal effect estimation; (6) CausalForestDML (CFDML)~\cite{athey2019generalized}: It employs two random forests for causal estimations for predicting two potential outcomes respectively. The VAE based estimators include: (1) causal effect variational autoencoder (CEVAE)~\cite{louizos2017causal} and (2) treatment effect by disentangled variational autoencoder (TEDVAE)~\cite{zhang2021treatment}.

\noindent \textbf{Evaluation metrics}. We employ the estimation bias $\left|(\hat{ATE}- ATE)/ATE \right|*100$ (\%) to evaluate the performance of all estimators, where ATE is the true causal effect and $\hat{ATE}$ is the estimated causal effect. We utilise the Precision of the Estimation of Heterogeneous Effect (PEHE) for the quality of CATE estimation~\cite{hill2011bayesian,louizos2017causal} defined as $\sqrt{\varepsilon_{PEHE}} = \sqrt{\mathbb{E}(((y_1 - y_0)-(\hat{y}_1 - \hat{y}_0))^2)}$, where $y_1, y_0$ represent the true potential outcomes and $\hat{y}_1, \hat{y}_0$ represent the predicted potential outcomes. Note that PEHE is widely employed for assessing CATE estimations  in causal inference\cite{guo2020survey}. To mitigate random noise, we repeat the experiments multiple times and report the average and the standard deviation. For the three real-world datasets, since there is no ground truth causal effects available, we evaluate all estimators against the reference causal effects found in the literature.

\noindent \textbf{Implementation details}.  We use \textit{Python} and the libraries including \textit{pytorch}~\cite{paszke2019pytorch}, \textit{pyro}~\cite{bingham2019pyro} and \textit{econml} to implement our proposed DLRCE algorithm. The implementation of DLRCE is available at the anonymous site~\url{https://anonymous.4open.science/r/DLRCE-385A}. The implementations of LDML, SLDML, KernelDML, LDRLearner and CFDML are from the \textit{Python} package \textit{encoml}~\cite{battocchi2019econml}. The implementations of X-learner and R-learner are from the \textit{Python} package \textit{CausalML}~\cite{chen2020causalml}. The implementation of CEVAE is based on the \textit{Python} library \textit{pyro}~\cite{bingham2019pyro} and the implementations of TEDVAE is from the authors' GitHub\footnote{\url{https://github.com/WeijiaZhang/TEDVAE}}.

\begin{table*}[t]
	\centering
	\setlength\tabcolsep{12pt}
	\caption{Estimation bias (mean$\pm$standard deviation) over 30 independently repeated experiments on the synthetic datasets with $Y_{nonlin}$. The best result is marked in boldface. Our proposed DLRCE algorithm obtains the smallest bias.}
	\begin{tabular}{|c|ccccc|}
		\hline
		\multirow{2}{*}{Method} & \multicolumn{5}{c|}{Sample sizes}  \\ \cline{2-6} 
		& \multicolumn{1}{c|}{2k}          & \multicolumn{1}{c|}{4k}           & \multicolumn{1}{c|}{6k}         & \multicolumn{1}{c|}{8k}         & 10k        \\ \hline		
		LDML               & \multicolumn{1}{c|}{45.39$\pm$0.71}  & \multicolumn{1}{c|}{{47.02$\pm$0.49}}   & \multicolumn{1}{c|}{46.90$\pm$0.30} & \multicolumn{1}{c|}{45.58$\pm$0.32} & 43.97$\pm$0.16 \\ \hline
		SLDML         & \multicolumn{1}{c|}{46.79$\pm$0.76}  & \multicolumn{1}{c|}{47.14$\pm$0.42}   & \multicolumn{1}{c|}{47.12$\pm$0.30} & \multicolumn{1}{c|}{45.79$\pm$0.33} & 44.03$\pm$0.16 \\ \hline
		KernelDML               & \multicolumn{1}{c|}{54.30$\pm$0.94}  & \multicolumn{1}{c|}{61.93$\pm$0.74}   & \multicolumn{1}{c|}{63.45$\pm$0.50} & \multicolumn{1}{c|}{61.72$\pm$0.39} & 62.34$\pm$0.26 \\ \hline
X-learner               & \multicolumn{1}{c|}{33.08$\pm$0.69}  & \multicolumn{1}{c|}{30.12$\pm$0.71}   & \multicolumn{1}{c|}{34.62$\pm$0.20} & \multicolumn{1}{c|}{33.80$\pm$0.27} & 30.60$\pm$0.23 \\ \hline
R-learner  & \multicolumn{1}{c|}{26.13$\pm$0.43}  & \multicolumn{1}{c|}{23.73$\pm$0.32}   & \multicolumn{1}{c|}{25.64$\pm$0.21} & \multicolumn{1}{c|}{25.60$\pm$0.12} & 23.44$\pm$0.13 \\ \hline
		LDRlearner         & \multicolumn{1}{c|}{69.19$\pm$0.98}  & \multicolumn{1}{c|}{72.61$\pm$0.80}   & \multicolumn{1}{c|}{71.94$\pm$0.48} & \multicolumn{1}{c|}{70.28$\pm$0.33} & 69.55$\pm$0.21 \\ \hline
		CFDML         & \multicolumn{1}{c|}{59.94$\pm$0.84}  & \multicolumn{1}{c|}{57.61$\pm$0.41}   & \multicolumn{1}{c|}{53.87$\pm$0.37} & \multicolumn{1}{c|}{51.52$\pm$0.20} & 48.86$\pm$0.16 \\ \hline
		CEVAE                   & \multicolumn{1}{c|}{24.15$\pm$3.08}  & \multicolumn{1}{c|}{61.91$\pm$2.30}   & \multicolumn{1}{c|}{46.21$\pm$3.71} & \multicolumn{1}{c|}{47.07$\pm$4.44} & 41.37$\pm$5.48 \\ \hline
		TEDVAE                  & \multicolumn{1}{c|}{59.96$\pm$1.21}  & \multicolumn{1}{c|}{59.15$\pm$0.63}   & \multicolumn{1}{c|}{54.94$\pm$0.40} & \multicolumn{1}{c|}{52.26$\pm$0.27} & 48.39$\pm$0.17 \\ \hline
		DLRCE                   & \multicolumn{1}{c|}{\textbf{15.52$\pm$0.89}} & \multicolumn{1}{c|}{\textbf{16.58$\pm$7.70}} & \multicolumn{1}{c|}{\textbf{19.32$\pm$3.37}} & \multicolumn{1}{c|}{\textbf{10.57$\pm$0.59}} & \textbf{10.32$\pm$0.65} \\ \hline	
	\end{tabular}
	\label{tab:bias002}
\end{table*}

\begin{table*}[t]
	\centering
	\setlength\tabcolsep{12pt}
	\caption{Estimated PEHE (mean$\pm$standard deviation) over 30 independently repeated experiments on the synthetic datasets with $Y_{nonlin}$ for different methods. The best result is marked in boldface. Our proposed DLRCE algorithm obtains the smallest PEHE.}
	\begin{tabular}{|c|ccccc|}
		\hline
		\multirow{2}{*}{Method}   & \multicolumn{5}{c|}{Samples}                                                                               \\ \cline{2-6} 
		& \multicolumn{1}{c|}{2k} & \multicolumn{1}{c|}{4k}   & \multicolumn{1}{c|}{6k}   & \multicolumn{1}{c|}{8k}         & 10k              \\ \hline
		LDML        & \multicolumn{1}{c|}{1.65$\pm$0.07}  & \multicolumn{1}{c|}{{1.52$\pm$0.04}} & \multicolumn{1}{c|}{1.53$\pm$0.02} & \multicolumn{1}{c|}{1.46$\pm$0.02} & 1.39$\pm$0.01                      \\ \hline
		SLDM  & \multicolumn{1}{c|}{1.62$\pm$0.07}  & \multicolumn{1}{c|}{1.53$\pm$0.05} & \multicolumn{1}{c|}{1.54$\pm$0.02} & \multicolumn{1}{c|}{1.45$\pm$0.02} & 1.39$\pm$0.01                      \\ \hline
		KernelDML        & \multicolumn{1}{c|}{1.69$\pm$0.12}  & \multicolumn{1}{c|}{1.78$\pm$0.10} & \multicolumn{1}{c|}{1.92$\pm$0.05} & \multicolumn{1}{c|}{1.87$\pm$0.03} & 1.88$\pm$0.01                      \\ \hline
		X-learner     & \multicolumn{1}{c|}{6.23$\pm$0.04}  & \multicolumn{1}{c|}{6.22$\pm$0.03}   & \multicolumn{1}{c|}{6.17$\pm$0.01} & \multicolumn{1}{c|}{6.18$\pm$0.01} & 6.10$\pm$0.01 \\ \hline
	R-learner       & \multicolumn{1}{c|}{6.91$\pm$0.04}  & \multicolumn{1}{c|}{4.81$\pm$0.21}   & \multicolumn{1}{c|}{3.76$\pm$0.01} & \multicolumn{1}{c|}{3.25$\pm$0.01} & 2.83$\pm$0.01 \\ \hline
		LDRlearner  & \multicolumn{1}{c|}{2.37$\pm$0.09}  & \multicolumn{1}{c|}{2.27$\pm$0.09}   & \multicolumn{1}{c|}{2.31$\pm$0.03} & \multicolumn{1}{c|}{2.24$\pm$0.03} & 2.20$\pm$0.02 \\ \hline
		CFDML  & \multicolumn{1}{c|}{2.04$\pm$0.06}  & \multicolumn{1}{c|}{1.92$\pm$0.03} & \multicolumn{1}{c|}{1.89$\pm$0.02} & \multicolumn{1}{c|}{1.84$\pm$0.01} & 1.76$\pm$0.01                      \\ \hline
		CEVAE            & \multicolumn{1}{c|}{1.46$\pm$0.21}  & \multicolumn{1}{c|}{1.90$\pm$0.18} & \multicolumn{1}{c|}{1.51$\pm$0.29} & \multicolumn{1}{c|}{1.35$\pm$0.34} & 1.56$\pm$0.42                      \\ \hline
		TEDVAE           & \multicolumn{1}{c|}{2.07$\pm$0.09}  & \multicolumn{1}{c|}{1.82$\pm$0.09} & \multicolumn{1}{c|}{1.75$\pm$0.03} & \multicolumn{1}{c|}{1.63$\pm$0.02} & 1.54$\pm$0.02                      \\ \hline
		DLRCE            & \multicolumn{1}{c|}{\textbf{0.70$\pm$0.09}} & \multicolumn{1}{c|}{\textbf{0.76$\pm$0.64}} & \multicolumn{1}{c|}{\textbf{0.95$\pm$0.35}} & \multicolumn{1}{c|}{\textbf{0.55$\pm$0.08}} & \textbf{0.54$\pm$0.04}                      \\ \hline
	\end{tabular}
	\label{tab:pehe002}
\end{table*}

\subsection{Evaluations on Synthetic Datasets}
\label{subsec:synthe}
We use the causal DAG in Fig.~\ref{fig:simpledags} (d) to generate the synthetic datasets with sample sizes, 2k, 4k, 6k, 8k, and 10k for our experiments. In the causal DAG $\mathcal{G}$, $\mathbf{M}$ and $\mathbf{X}$ are two set of proxy variables. $L$, $F$ and $\mathbf{Z}$ are latent confounders. Similar to~\cite{hassanpour2019counterfactual,cheng2022toward}, $L$, $F$ and $\mathbf{Z}$ are generated from Bernoulli distribution. For an element $M\in \mathbf{M}$, it is generated from the two latent confounders $L$ and $F$ by using $M = \eta_1*L + \eta_2*F$, where $\eta_1$ and $\eta_2$ are two coefficients.  For an element $X\in \mathbf{X}$, it is generated from the latent confounder $Z$ by using $X\sim N(Z, \eta_3*Z)$, where $\eta_3$ is a coefficient. For generating the treatment $W$, we use Bernoulli distribution with the conditional probability $P(W=1\mid L, Z, \mathbf{M}) = [1+exp \{1+0.25*L+0.25*Z\}]$. 

In this work, we generate two types of potential outcomes $Y(W)$, namely a linear function $Y_{linear}$ and a nonlinear function $Y_{nonlinr}$ as $Y(W) = 2 + 3*W + 3*\mathbf{M} + 2*F*\mathbf{M}+ 3*Z  +\epsilon_{w}$, where $\epsilon_w$ is an error term, and $Y(W) = 2 + 3*W + L*\mathbf{M} + \mathbf{M} + 2*F + 3*Z  +\epsilon_{w}$, respectively. Based on the data generation process,  all synthetic datasets have both potential outcomes, \ie the true ITE for an individual is known. In our simulation study,  the true ATE is $3$. To evaluate the performance of our DLRCE algorithm, we conduct the experiments 30 times independently for each setting.	

We report the estimation bias and PEHE for the synthetic datasets generated from $Y_{linear}$ in Tables~\ref{tab:bias001} and~\ref{tab:pehe001}, and for the synthetic datasets generated from $Y_{nonlin}$ in Tables~\ref{tab:bias002} and~\ref{tab:pehe002}.

\noindent \textbf{Results.}  From the experimental results, we have the following observations: (1) Machine learning based estimators, LDML, SLDML, KernelDML, X-learner, R-learner, LDRlearner and CFDML have a large estimation bias and PEHE on both types of synthetic datasets since these estimators rely on the assumption of unconfoundedness and cannot learn a valid representation from proxy variables to block all back-door paths between $W$ and $Y$. (2) VAE based estimators, TEDVAE and CEVAE methods have a large estimation bias and PEHE on both types of synthetic datasets since both methods fail to deal with the confounding $M$-bias variable studied in this work. (3) The proposed DLRCE algorithm obtains the smallest estimation bias and PEHE  among all methods  on both types of synthetic datasets since our DLRCE algorithm learns and disentangles three latent representations $\mathbf{Z}$, $\mathbf{L}$ and $\mathbf{F}$ from proxy variables $(\mathbf{X, M})$ to effectively block all back-door paths between $W$ and $Y$. The smallest estimation bias and PEHE further confirm the correctness of our DLRCE algorithm in learning three latent representations $\mathbf{Z}$, $\mathbf{L}$ and $\mathbf{F}$ from proxy variables. (4) The compared algorithms, machine learning based and VAE based estimators achieve better performance compared to the synthetic datasets generated from $Y_{linear}$ and relatively poorer performance on the synthetic datasets generated from $Y_{nonlin}$. Our proposed DLRCE algorithm consistently produces good performance across both types of datasets.  
	 
	In sum, the simulation studies demonstrate that the proposed DLRCE algorithm  effectively addresses the problem of confounding M-bias when estimating ATE and CATE from observational data in the presence of latent confounders. It further provides evidence that DLRCE is capable of recovering latent variable representations from proxy variables.

\subsection{Parameters Analysis}
	In our DLRCE algorithm, there are three tuning parameters, namely $\alpha$, $\beta$, and $\gamma$, used to balance $\mathcal{L}_{ELBO}$ and the three classifiers during the training process. We consider setting $\{\alpha, \beta, \gamma\} = \{0.1, 0.5, 1, 1.5, 2\}$ to analyse the sensitivity of the three parameters on synthetic datasets with a sample size of 10k, generated using the same data generation process described in Section~\ref{subsec:synthe}. We report the estimation bias of DLRCE algorithm in Table~\ref{tab:parasets}. From Table~\ref{tab:parasets}, we observe that the three parameters ${\alpha, \beta, \gamma}$ have a low sensitivity to the estimation bias of the DLRCE algorithm in ATE estimation. In summary, it is recommended to set the three tuning parameters, $\alpha$, $\beta$, and $\gamma$, to small values for our DLRCE algorithm.

	\begin{table}[t]
		\centering
		\setlength\tabcolsep{10pt}
		\caption{The estimation bias with the different setting of tunning parameters $\alpha$, $\beta$ and $\gamma$.}
		\begin{tabular}{|c|cc|}
			\hline
			\multirow{2}{*}{Weight} & \multicolumn{2}{c|}{Dataset}                     \\ \cline{2-3} 
			& \multicolumn{1}{c|}{Linear}       & Nonlinear    \\ \hline
			$\{\alpha,\beta,\gamma\} = 0.1$ & \multicolumn{1}{c|}{14.23$\pm$0.54}   & 9.31$\pm$0.57   \\ \hline
			$\{\alpha,\beta,\gamma\} = 0.5$ & \multicolumn{1}{c|}{14.46$\pm$1.26}   & 12.92$\pm$1.33   \\ \hline
			$\{\alpha,\beta,\gamma\} = 1$ & \multicolumn{1}{c|}{13.59$\pm$0.54}   & 10.32$\pm$0.65   \\ \hline
			$\{\alpha,\beta,\gamma\} = 1.5$ & \multicolumn{1}{c|}{11.60$\pm$0.88}   & 18.18$\pm$2.12   \\ \hline
			$\{\alpha,\beta,\gamma\} = 2$ & \multicolumn{1}{c|}{15.15$\pm$0.59}   & 11.03$\pm$1.00   \\ \hline
		\end{tabular}
		\label{tab:parasets}
	\end{table}

\subsection{A Study on the Dimensionality of Latent Representations}
In our simulation studies, we set the dimensions of $\mathbf{L}$, $\mathbf{F}$, and $\mathbf{Z}$ to 1, respectively. We conducted a study on the dimensionality of latent representations to demonstrate the effectiveness of this setting. To achieve this goal, we fixed the sample size to 10k for all synthetic datasets and repeated the experiments 30 times independently to minimise random noise for each setting. Following the data generation process described in Section~\ref{subsec:synthe}, we generated a set of synthetic datasets with dimensions of the three latent variables $(\mathbf{L}$, $\mathbf{F}$,  $\mathbf{Z})$ set to $\{1, 3, 5, 7, 9\}$ respectively. In our DLRCE algorithm, we set three parameters $(|\mathbf{L}|$, $|\mathbf{F}|$, $|\mathbf{Z}|)$ to $\{1, 3, 5, 7, 9\}$ respectively to conduct experiments on these synthetic datasets. The estimation bias of the DLRCE algorithm on these datasets is displayed in Fig.~\ref{fig:dims}. From Fig.~\ref{fig:dims}, we observe that the estimation bias of the DLRCE algorithm is the smallest on both types of synthetic datasets when $(|\mathbf{L}|$, $|\mathbf{F}|$, $|\mathbf{Z}|)$  is set to (1, 1, 1) regardless of the true dimensions of $\mathbf{L}$, $\mathbf{F}$, and $\mathbf{Z}$ in the data. Hence, this finding suggests that setting $|\mathbf{L}|$, $|\mathbf{F}|$, and $|\mathbf{Z}|$ to 1 is reasonable.

\begin{figure}[t]
		\centering
		\includegraphics[scale=0.265]{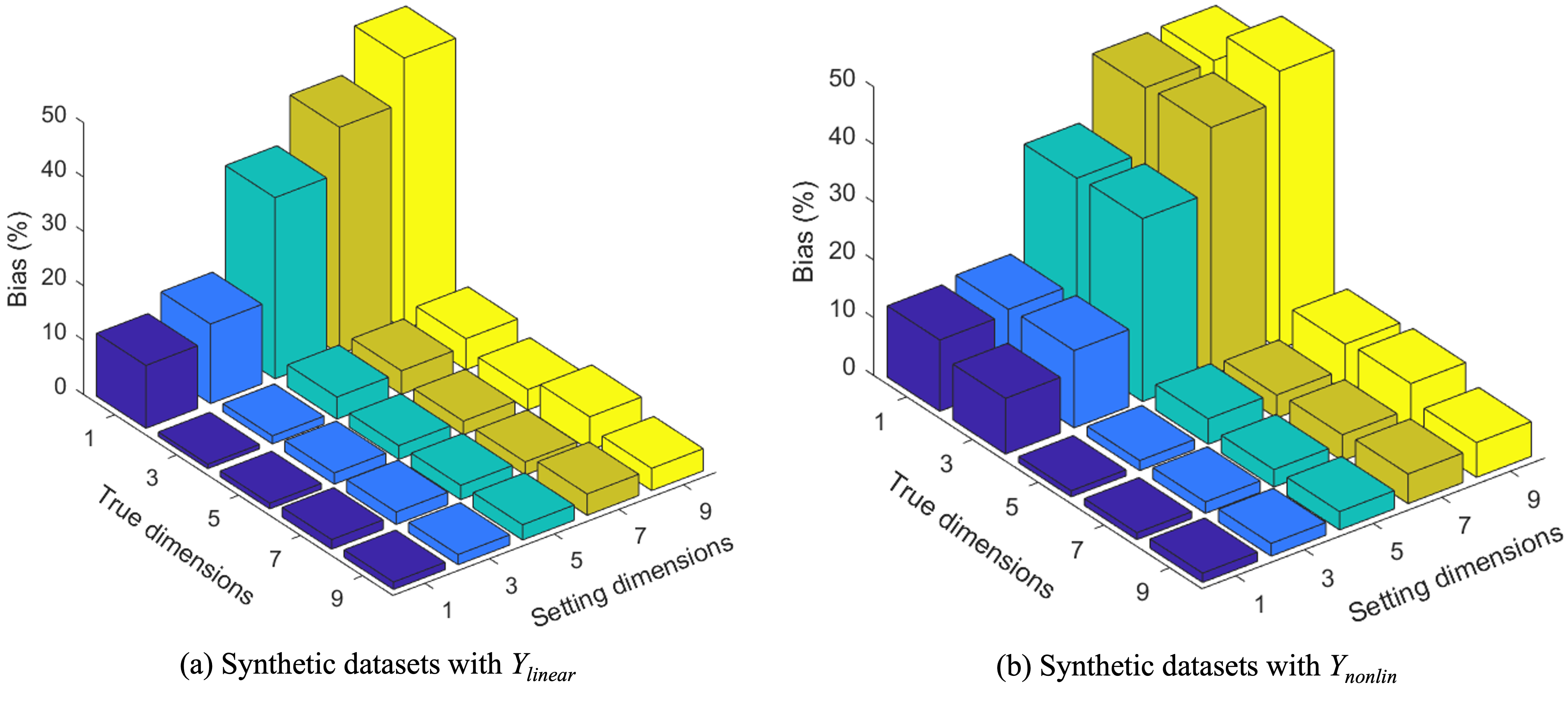}
		\caption{Estimation bias on both types of synthetic datasets. `True dimensions' refer to the dimensions of $\mathbf{L}$, $\mathbf{F}$, and $\mathbf{Z}$ in the data, and `Setting dimensions' correspond to the parameters of  $|\mathbf{L}|$, $|\mathbf{F}|$, and $|\mathbf{Z}|$ in the DLRCE algorithm. }
		\label{fig:dims}
	\end{figure} 
	
\subsection{Ablation Study}
	Next, we examine the impact of three latent representations $\mathbf{L}$, $\mathbf{F}$, and $\mathbf{Z}$ on the performance of DLRCE. To do this, we set the dimensions of $(\mathbf{L}, \mathbf{F}, \mathbf{Z})$ to (1,0,0), (0,1,0), (0,0,1), (1,1,0), (1,0,1), (0,1,1), and (1,1,1), respectively. We conduct a series of experiments on both types of synthetic datasets with a sample size of 10k, generated using the same data generation process described in Section~\ref{subsec:synthe}. Figure~\ref{fig:radar} illustrates the capability of each latent representation in terms of estimation bias using a radar chart. For example, in Figure~\ref{fig:radar} (a), the DLRCE performances achieve the smallest estimation bias when the dimensions of $(\mathbf{L}, \mathbf{F}, \mathbf{Z})$ are set to (1,1,1). It is worth noting that $(\mathbf{L}, \mathbf{F}, \mathbf{Z})=(0,1,1)$ yields the second smallest estimation bias, consistent with the conclusion in Theorem~2. Moreover, all three latent representations contribute to bias reduction, with $\{\mathbf{Z, F}\}$ contributing the most.

\begin{figure}[t]
	\centering
	\includegraphics[scale=0.388]{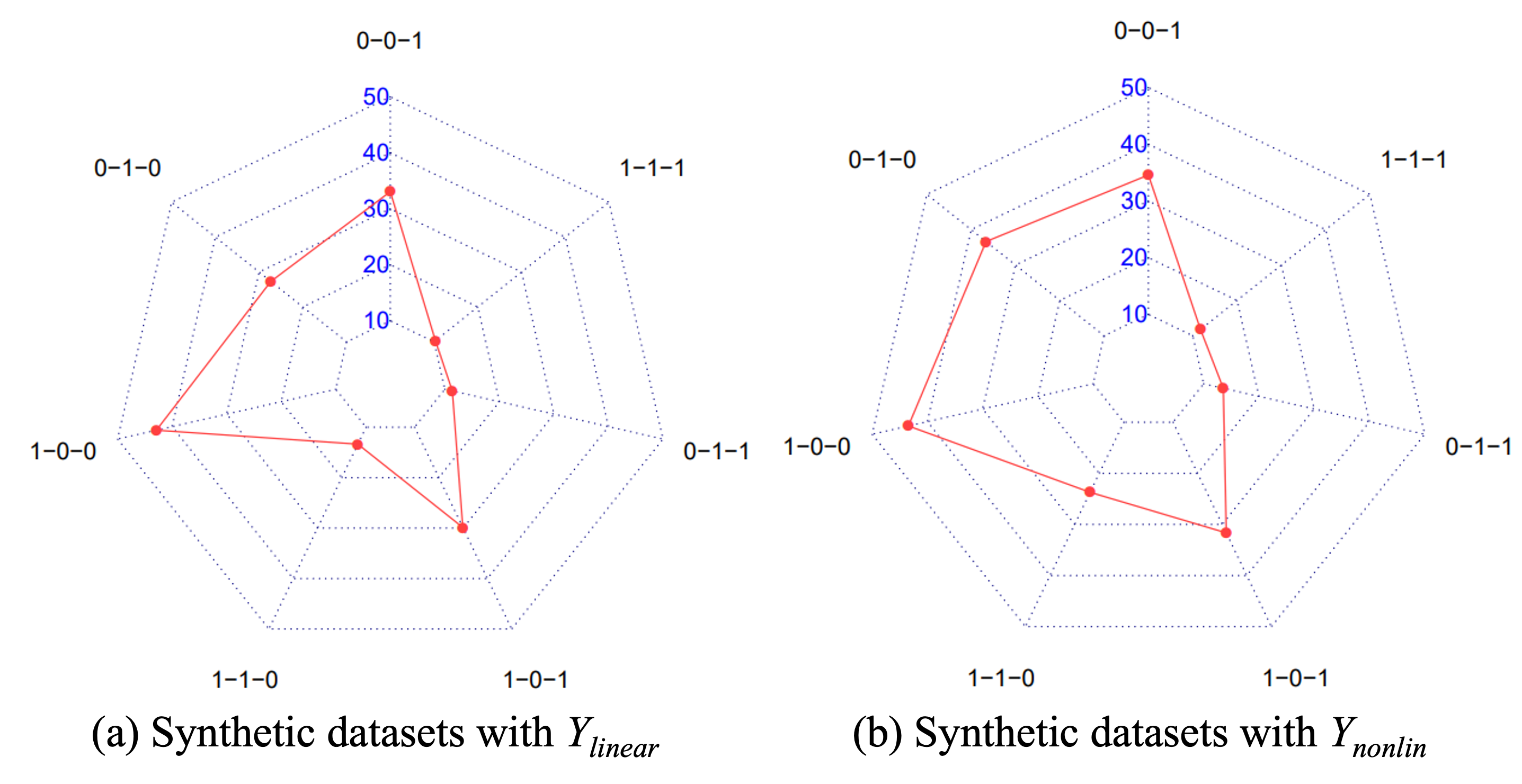}
	\caption{Radar charts for DLRCE’s capability in learning and disentangling the latent representations. Each vertex on the polygons denotes the latent representations’ dimensions. For example, 0-1-1 implies that $(\mathbf{L}, \mathbf{F}, \mathbf{Z})$=(0, 1, 1), \ie $|\mathbf{L}|=0$. }
\label{fig:radar}
\end{figure}

\subsection{Experiments on Three Real-World Datasets}
	In this section, we assess the performance of DLRCE against the above-mentioned comparisons on three real-world datasets, Schoolingreturns~\cite{card1993using}, Cattaneo2~\cite{cattaneo2010efficient} and Sachs~\cite{sachs2005causal} for which the empirical causal effects are available in the literature. The details of the three datasets are described below. 
	
\noindent \textbf{Schoolingreturns Dataset}.
	This dataset consists of 3,010 records and 19 variables~\cite{card1993using}. The treatment variable is the education level of a person. The outcome variable is raw wages in 1976 (in cents per hour). The goal of collecting this dataset is to study the causal effect of the education level on wages.  The estimated $ATE(W, Y) = 0.1329$ with 95\% confidence interval (0.0484, 0.2175) from the works~\cite{verbeek2008guide} as the reference causal effect.  

\noindent \textbf{Cattaneo2 Dataset}. 
	The Cattaneo 2~\cite{cattaneo2010efficient} is widely employed to investigate the ATE of maternal smoking status during pregnancy ($W$) on a baby's birth weight (in grams)\footnote{It can be downloaded from the site: \url{http://www.stata-press.com/data/r13/cattaneo2.dta}}.
	Cattaneo2 consists of the birth weights of 4,642 singleton births in Pennsylvania, USA~\cite{almond2005costs,cattaneo2010efficient}. Cattaneo2 contains 864 smoking mothers ($W$=1) and 3,778 nonsmoking mothers ($W$=0). The dataset contains several covariates: mother's age, mother's marital status, an indicator for the previous infant where the newborn died, mother's race, mother's education, father's education, number of prenatal care visits, months since last birth, an indicator of firstborn infant and an indicator of alcohol consumption during pregnancy. The authors~\cite{almond2005costs} found a strong negative effect of maternal smoking on the weights of babies, namely about $200g$ to $250g$ lighter for a baby with a mother smoking during pregnancy by statistical analysis on all covariates.

\noindent \textbf{Sachs Dataset}.
	The dataset contains 853 samples and 11 variables~\cite{sachs2005causal}. The treatment is $Erk$ (the manipulation of concentration levels of a molecule). The outcome is the concentration of $Akt$. In this work, we take the reported $ATE(W, Y) =1.4301$  with 95\% confidence interval (0.05, 3.23) in the work~\cite{cheng2023discovering} as the reference causal effect.

\begin{table}[t]
\centering
\setlength\tabcolsep{10pt}
  \caption{The estimated $\hat{ATE}$ on the three real-world datasets. Note that the estimated $\hat{ATE}$ by our DLRCE falls in the empirical causal effect interval on all three real-world datasets. The estimated $\hat{ATE}$ is within the empirical interval marked in boldface.}
	\begin{tabular}{|c|ccc|}
			\hline
			\multirow{2}{*}{Method} & \multicolumn{3}{c|}{Datasets}                                              \\ \cline{2-4} 
			& \multicolumn{1}{c|}{Schoolingreturns} & Cattaneo2  & \multicolumn{1}{|c|}{Sachs}  \\ \hline	
LDML    & \multicolumn{1}{c|}{-0.045}    & -170.179   & \multicolumn{1}{|c|}{36.118}   \\ \hline
SLDML   & \multicolumn{1}{c|}{-0.504}    & -153.859  & \multicolumn{1}{|c|}{152.900}  \\ \hline
KDML          & \multicolumn{1}{c|}{-0.021}    & -146.824   & \multicolumn{1}{|c|}{19.360}  \\ \hline
    X-Learner & \multicolumn{1}{c|}{\textbf{0.161}}  & \textbf{-230.61}   & \multicolumn{1}{|c|}{18.661}  \\ \hline	
    R-Learner & \multicolumn{1}{c|}{-0.020}  & \textbf{-234.96}   & \multicolumn{1}{|c|}{24.072}  \\ \hline
      LDRLearner          & \multicolumn{1}{c|}{-0.020}    & -179.853   & \multicolumn{1}{|c|}{37.400}  \\ \hline
		CFDML             & \multicolumn{1}{c|}{-0.040}         & \textbf{-241.436}         & \multicolumn{1}{|c|}{25.774}   \\ \hline
			CEVAE                   & \multicolumn{1}{c|}{0.026}     & \textbf{-221.234}  & \multicolumn{1}{|c|}{\textbf{0.254}}   \\ \hline
			TEDVAE                & \multicolumn{1}{c|}{0.231}     & \textbf{-235.325}     & \multicolumn{1}{|c|}{\textbf{0.255}}   \\ \hline
			DLRCE                     & \multicolumn{1}{c|}{\textbf{0.101}}     & \textbf{-226.448}    & \multicolumn{1}{|c|}{\textbf{1.278}} \\ \hline
		\end{tabular}
  \label{tab:realdata}
	\end{table}

\noindent  \textbf{Results.} We report the results on the three real-world datasets in Table~\ref{tab:realdata}. From Table~\ref{tab:realdata}, we can see that (1) the estimated $\hat{ATE}$s by DLRCE on three real-world datasets are within the empirical intervals respectively. (2) The estimated $\hat{ATE}$ by X-learner on Schoolingreturns, by X-learner, R-learner, CFDML, CEVAE and TEDVAE on Cattaneo2, and by CEVAE and TEDVAE on Sachs are within the empirical intervals, but these methods do not produce estimates within the confidence intervals for all three data sets.  The other methods fail to obtain an estimation within the empirical intervals on any of the  three datasets. (3) The estimates of LDML, SLDML, KDML, R-learner, LDRLearner and CFDML on Schoolingreturns are negative which is opposite to a positive estimate in the literature~\cite{verbeek2008guide}. (4) The estimated $\hat{ATE}$s on Sachs by Machine learning based estimators, such as LDML, SLDML, KDML, X-learner, R-learner, LDRLearner and CFDML, are far away from the empirical interval (0.05, 3.23).

 In a word, the proposed DLRCE algorithm performs better than the stat-of-the-art causal effect estimators on the three real-world datasets. This further confirms the potential applicability of DLRCE in real-world applications.

\section{Related Work}
\noindent	\textbf{Machine learning for causal effect estimation}. Causal effect estimations from observational data have received extensive attention from the artificial intelligence and statistics communities~\cite{pearl2009causality,imbens2015causal,hernan2020causal,guo2020survey,cheng2022data}. For instance, matching methods~\cite{rubin1996matching,abadie2006large} and tree-based methods~\cite{chipman2010bart,athey2016recursive,athey2019generalized,wager2018estimation} have been developed to address confounding bias in causal effect estimation from observational data. Additionally, meta-learners~\cite{kunzel2019metalearners} have also been studied for estimating the average treatment effect (ATE) and conditional average treatment effect (CATE) from observational data.
	
\noindent	\textbf{Representation learning for causal effect estimation}. Recently, representation learning methods~\cite{shalit2017estimating,yoon2018ganite,guo2020survey} have been applied to causal effect estimation, but they often rely on the unconfoundedness assumption~\cite{imbens2015causal}. For example, Shalit et al. proposed a balanced representation learning method for counterfactual regression (CFRNet)~\cite{shalit2017estimating}. Yoon \etal first used a GAN model to learn representations for causal effect estimation. Different from these methods, our proposed DLRCE algorithm addresses the challenging problem of confounding $M$-bias variable in causal effect estimation.
	
\noindent \textbf{Proxy variables for causal effect estimation}. Proxy variables are the measured covariates that are at best of the true underlying confounding mechanism~\cite{louizos2017causal,kallus2018causal,miao2018identifying}.  Kallus \etal~\cite{kallus2018causal} proposed to infer the confounders from proxy variables by using matrix factorisation. Miao \etal~\cite{miao2018identifying} proposed the general conditions for causal effects identification using more general proxies, but they did not propose a practical data-driven method. CEVAE~\cite{louizos2017causal} uses the VAE model to learn the representations from proxy variables for causal effect estimation. However, CEVAE fails to deal with the confounding $M$-bias problem in data studied in this work as shown in our experiments. To the best of our knowledge, our DLRCE algorithm is the first work to solve the problem of confounding $M$-bias variable using the disentanglement of representation learning techniques.  
	
\section{Conclusion}
\label{sec:discon}
In this paper, we identify a challenging problem in estimating causal effects from observational data in the presence of latent confounders, \ie  the problem of confounding $M$-bias as shown in the causal DAG in Fig.~\ref{fig:simpledags} (c). Existing methods tackle confounding bias through balanced representation learning or covariate adjustment, but are unable to handle the problem of confounding $M$-bias, and lead to biased causal effect estimation as shown in our experiments.  To address this problem, we propose a novel disentangled representation learning framework, the DLRCE algorithm for causal effect estimation from observational data in the presence of latent confounders. DLRCE learns three sets of latent representations from proxy variables to adjust for both confounding bias and $M$-bias. Extensive experiments on synthetic and three real-world datasets demonstrate that DLRCE outperforms existing causal effect estimation methods for ATE  and CATE estimation in datasets with both types of biases. The proposed method shows promise in causal effect estimation in real-world datasets and opens up avenues for addressing complex confounding scenarios in causal inference.
 	
\section{acknoledgement}
This work was supported by the Chinese Central Universities (grant number: 2662023XXPY004) and the Australian Research Council (grant number: DP230101122).
		
\bibliographystyle{IEEEtran}
\bibliography{icdm2023}

\end{document}